\documentclass[lettersize,journal]{IEEEtran}
\usepackage{booktabs}
\usepackage{amsmath}
\usepackage{balance}
\usepackage{enumitem}
\usepackage{amsthm}
\usepackage{hyperref}
\usepackage{amssymb}
\usepackage{colortbl}
\usepackage{stfloats}
\usepackage{xcolor} 
\usepackage{textcomp}
\usepackage{graphicx}
\usepackage{verbatim}
\usepackage{multirow}
\usepackage{algorithm}
\usepackage{multicol}
\usepackage{algpseudocode}
\usepackage{mathtools}
\newtheorem{theorem}{Theorem}
\newtheorem{lemma}[theorem]{Lemma}
\newtheorem{definition}{Definition}

\begin{document}

\title{Multi-Objective Optimization for Privacy-Utility Balance in Differentially Private Federated Learning}

\author{Kanishka Ranaweera,~\IEEEmembership{Student Member, ~IEEE, } David Smith,~\IEEEmembership{Member, ~IEEE, } Pubudu N. Pathirana,~\IEEEmembership{Senior Member, ~IEEE, } Ming Ding,~\IEEEmembership{Senior Member, ~IEEE, } Thierry Rakotoarivelo,~\IEEEmembership{Senior Member, ~IEEE, } Aruna Seneviratne,~\IEEEmembership{Senior Member, ~IEEE }

\IEEEcompsocitemizethanks{\IEEEcompsocthanksitem Kanishka Ranaweera is with School of Engineering and Built Environment, Deakin University, Waurn
Ponds, VIC 3216, Australia, and also with the Data61, CSIRO, Eveleigh, NSW 2015, Australia. \protect\\
E-mail: kranaweera@deakin.edu.au
\IEEEcompsocthanksitem David Smith is with Data61, CSIRO, Eveleigh, NSW 2015, Australia. \protect\\
E-mail: david.smith@data61.csiro.au
\IEEEcompsocthanksitem Pubudu N. Pathirana is with School of Engineering and Built Environment, Deakin University, Waurn Ponds, VIC 3216, Australia. \protect\\
E-mail: pubudu.pathirana@deakin.edu.au
\IEEEcompsocthanksitem Ming Ding is with Data61, CSIRO, Eveleigh, NSW 2015, Australia. \protect\\
E-mail: ming.ding@data61.csiro.au
\IEEEcompsocthanksitem Thierry Rakotoarivelo is with Data61, CSIRO, Eveleigh, NSW 2015, Australia. \protect\\
E-mail: thierry.rakotoarivelo@data61.csiro.au
\IEEEcompsocthanksitem Aruna Seneviratne is with School of Electrical Engineering and Telecommunications, University of New South Wales (UNSW), NSW, Australia. \protect\\
E-mail: a.seneviratne@unsw.edu.au

}
}

\maketitle

\begin{abstract}
Federated learning (FL) enables collaborative model training across distributed clients without sharing raw data, making it a promising approach for privacy-preserving machine learning. However, ensuring differential privacy (DP) in FL presents challenges due to the trade-off between model utility and privacy protection. Clipping gradients before aggregation is a common strategy to limit privacy loss, but selecting an optimal clipping norm is non-trivial, as excessively high values compromise privacy, while overly restrictive clipping degrades model performance. In this work, we propose an adaptive clipping mechanism that dynamically adjusts the clipping norm using a multi-objective optimization framework. By integrating privacy and utility considerations into the optimization objective, our approach balances privacy preservation with model accuracy. We theoretically analyze the convergence properties of our method and demonstrate its effectiveness through extensive experiments on MNIST, Fashion-MNIST, and CIFAR-10 datasets. Our results show that adaptive clipping consistently outperforms fixed-clipping baselines, achieving improved accuracy under the same privacy constraints. This work highlights the potential of dynamic clipping strategies to enhance privacy-utility trade-offs in differentially private federated learning.

\end{abstract}


\section{Introduction}

Federated Learning (FL) has emerged as a transformative paradigm for collaborative training of machine learning models without centralized data aggregation \cite{mcmahan2017communication, li2020federated}. This distributed approach is particularly appealing for privacy-sensitive applications such as healthcare, where hospitals collaboratively train diagnostic models without exposing patient records\cite{nguyen2022federated}; finance, where multiple banks develop fraud detection algorithms while keeping customer data private\cite{long2020federated}; and mobile applications, where devices personalize predictive keyboards without sharing user text inputs with a central server\cite{hard2018federated}. FL enhances privacy by ensuring that raw data never leaves local devices or institutions, reducing the risk of direct data exposure. However, despite this advantage, FL alone does not guarantee strong privacy protection, as model updates exchanged between clients and the central server can still leak sensitive information through membership inference attacks or model inversion techniques \cite{hu2022membership,fredrikson2015model}. Thus, additional privacy-preserving mechanisms, such as Differential Privacy (DP)\cite{dwork2008differential,dwork2014algorithmic}, are essential to mitigate these risks.

DP is a formal framework that limits the information that can be inferred about individual data points from a dataset. When applied to FL, DP ensures that an adversary, even with access to model updates, cannot confidently determine whether a specific client's data was used during training. This is typically achieved by injecting controlled noise into model updates, thereby obscuring individual contributions. However, integrating DP into FL introduces unique challenges, particularly in balancing privacy guarantees with model utility \cite{el2022differential}. A key factor in this balance is the clipping norm used in DP-SGD, which directly impacts the amount of noise added to gradients. A higher clipping norm retains more gradient information but requires greater noise to maintain privacy guarantees, which can degrade model performance. Conversely, a lower clipping norm limits the noise required but risks discarding essential signal information, leading to undertrained models \cite{zhang2022understanding}. Addressing this trade-off is crucial for ensuring that privacy-preserving FL frameworks remain both effective and practical across diverse real-world applications.

In this paper, we propose a novel approach to address this trade-off by dynamically optimizing the clipping norm during training. Unlike static or manually tuned clipping norms, our method leverages a multi-objective optimization (MOO) framework that adjusts the clipping norm adaptively at each epoch. This optimization is driven by a combined objective that simultaneously minimizes privacy loss and maximizes model utility. By aligning the clipping norm with the training dynamics, our method ensures that gradients are appropriately clipped to achieve high utility while respecting rigorous privacy guarantees.

Our key contributions are as follows:
\begin{itemize}
    \item We introduce a novel MOO framework to guide clipping norm adjustment based on observed training dynamics, ensuring scalable and effective integration into existing FL pipelines.
    \item We present a theoretical convexity analysis to establish the mathematical properties of our optimization framework, ensuring a well-defined solution space for adaptive clipping.
    \item We provide a rigorous convergence analysis to demonstrate the effectiveness and robustness of our approach, ensuring that dynamic clipping does not hinder model optimization.
    \item We validate our method on commonly used benchmark datasets, including MNIST, FMNIST, and CIFAR-10, to demonstrate its effectiveness and generalizability across various data distributions and models.
\end{itemize}

The rest of the paper is organized as follows: In Section \ref{sec:background}, we provide an overview of related work on differentially private federated learning (DP-FL) and adaptive mechanisms. Section \ref{sec:adv_model} describes the adversarial model. Section \ref{sec:method} details the proposed method, including the MOO framework and the theoretical convergence analysis of our approach. In Section \ref{sec:results}, we describe the experimental setup, datasets, followed by a discussion of results and insights. Finally, Section \ref{sec:conclusion} concludes the paper with a summary of findings and directions for future work.

\section{Background}\label{sec:background}
This section provides an overview of the key concepts and methodologies relevant to DP-FL.

\subsection{Federated Learning}

FL \cite{mcmahan2017communication} is a distributed machine learning paradigm that enables multiple clients, such as edge devices or organizations, to collaboratively train a global model without sharing their private data. This approach ensures data privacy and security by keeping the data localized while only exchanging model updates with a central server.

\textbf{Problem Definition: }
Let $\mathcal{D}_k = \{(x_i^k, y_i^k)\}_{i=1}^{n_k}$ represent the local dataset held by client $k$, where $x_i^k \in \mathbb{R}^d$ denotes the input features and $y_i^k \in \mathbb{R}$ denotes the corresponding label. The objective of FL is to minimize a global loss function $L(\theta)$, defined as the weighted sum of local loss functions:
\begin{equation}
\label{eq:1}
    L(\theta) = \sum_{k=1}^K \frac{n_k}{n_{\mathcal{S}_t}} L_k(\theta),
\end{equation}
where $L_k(\theta)$ is the local loss function for client $k$ given by
\begin{equation}
    L_k(\theta) = \frac{1}{n_k} \sum_{i=1}^{n_k} \ell(f(x_i^k; \theta), y_i^k).
\end{equation}
Here, $\theta \in \mathbb{R}^d$ represents the global model parameters, $\ell(\cdot, \cdot)$ is a loss function (e.g., mean squared error or cross-entropy), $K$ is the number of clients, $n_k$ is the number of data points held by client $k$, and $n_{\mathcal{S}_t} = \sum_{k \in \mathcal{S}_t} n_k$ is the total number of data points across all clients.

 FL is designed to perform efficient and communication-aware distributed optimization. The algorithm proceeds as follows:

\begin{enumerate}
    \item \textbf{Initialization:} The central server initializes the global model parameters $\theta_0$.
    \item \textbf{Client Selection:} In each round $t$, a subset $\mathcal{S}_t \subseteq \{1, \ldots, K\}$ of clients is selected to participate in training.
    \item \textbf{Local Training:} Each selected client $k \in \mathcal{S}_t$ updates the global model using its local dataset $\mathcal{D}_k$. The local update is computed by solving the following optimization problem using $E$ epochs of stochastic gradient descent (SGD):
    \begin{equation}
        \theta_{t+1}^k = \theta_t^k - \eta_t \nabla L_k( \theta_t^k),
    \end{equation}
    where $\eta_t$ is the learning rate.
    \item \textbf{Aggregation:} The central server aggregates the updated models received from the clients to compute the new global model:
    \begin{equation}
        \theta_{t+1} = \sum_{k \in \mathcal{S}_t} \frac{n_k}{n_{\mathcal{S}_t}} \theta_{t+1}^k,
    \end{equation}
\end{enumerate}

FL provides several advantages, including enhanced data privacy, reduced communication overhead, and the ability to train models on diverse and distributed datasets\cite{kairouz2021advances}. However, it also introduces significant challenges, such as:
\begin{itemize}
    \item \textbf{Heterogeneous Data:} Clients may have non-IID (non-independent and identically distributed) data, leading to training instability.
    \item \textbf{System Heterogeneity:} Clients may have varying computational and communication capabilities.
    \item \textbf{Privacy Concerns:} Although raw data is not shared, model updates can still leak sensitive information, necessitating privacy-preserving mechanisms such as DP \cite{dwork2006}.
\end{itemize}

\subsection{Differential Privacy}

DP\cite{dwork2006,dwork2014algorithmic,dwork2008differential} is a rigorous mathematical framework for quantifying privacy guarantees. In the context of FL, DP ensures that individual clients' data cannot be inferred from the shared model updates, even by an adversary with access to auxiliary information\cite{mcmahan2017learning}. This section explores the integration of DP into FL, the mechanisms employed, and the role of privacy accountants \cite{abadi2016deep,wang2019subsampled} in measuring the cumulative privacy loss.

\begin{definition}
A randomized mechanism $\mathcal{M}: \mathcal{D} \to \mathcal{R}$ is said to provide $(\varepsilon, \delta)$-DP if, for any two datasets $D, D'$ differing in at most one data point and for any subset of outputs $S \subseteq \mathcal{R}$, the following holds:
\begin{equation}
    \Pr[\mathcal{M}(D) \in S] \leq e^{\varepsilon} \Pr[\mathcal{M}(D') \in S] + \delta.
\end{equation}
Here, $\varepsilon > 0$ is the privacy parameter, where smaller values of $\varepsilon$ correspond to stronger privacy guarantees, and $\delta \geq 0$ represents the probability of a privacy breach, allowing for a small relaxation of the guarantee. The case where $\delta = 0$ corresponds to pure $\varepsilon$-DP.
\end{definition}

To achieve DP in FL, two primary mechanisms are commonly employed:
\begin{itemize}
    \item \textbf{Noise Addition:} Gaussian noise is added to the model updates or aggregated results to obscure individual contributions. Noise can be introduced in two key ways:
    \begin{enumerate}
        \item \textbf{User-Level DP:} In this approach, noise is added at the aggregation stage on the server side. After collecting model updates from clients, the server perturbs the aggregated result to achieve DP\cite{abadi2016deep}. Mathematically, the noisy aggregation can be expressed as:
        \begin{equation}
            \theta_{t+1} = \sum_{k \in \mathcal{S}_t} \frac{n_k}{n_{\mathcal{S}_t}} \theta_{t+1}^k + \mathcal{N}(0, \sigma^2),
        \end{equation}
        where $\mathcal{N}(0, \sigma^2)$ is Gaussian noise with variance $\sigma^2$, and the variance $\sigma^2$ is proportional to the square of the clipping norm $C$. This dependence ensures that the noise scales appropriately with the sensitivity of the clipped updates.

        \item \textbf{Sample-Level DP:} In this approach, noise is added during local training at the client level\cite{mcmahan2017learning}. Each client perturbs their gradient updates before sending them to the server. The local update with noise can be expressed as:
        \begin{equation}
            \theta_{t+1}^k = \theta_{t+1}^k + \mathcal{N}(0, \sigma^2).
        \end{equation}
        Here, the noise variance $\sigma^2$ is also proportional to the square of the clipping norm $C$ at the client level.
        This ensures that the privacy guarantee holds locally for each client update while respecting the overall sensitivity bound imposed by the clipping threshold.
    \end{enumerate}

    \item \textbf{Clipping:} To limit the sensitivity of the gradients, individual client updates are clipped before aggregation:
    \begin{equation}
        \theta_{t+1}^k \gets \theta_{t+1}^k \cdot \min\left(1, \frac{C}{\|\theta_{t+1}^k\|}\right),
    \end{equation}
    where $C$ is a predefined clipping threshold. Clipping plays a dual role in FL. Firstly, it ensures that updates with excessively large magnitudes do not dominate the aggregation, thus stabilizing training and reducing the impact of outliers. Secondly, it bounds the sensitivity of updates, which is critical for accurately calibrating the noise required to achieve DP. 

    The choice of the clipping threshold $C$ introduces a trade-off between privacy and utility. If $C$ is too small, important updates may be excessively suppressed, leading to a loss in model accuracy. Conversely, a large $C$ increases the sensitivity, requiring more noise to maintain privacy guarantees, which can degrade the utility of the global model. Adaptive clipping mechanisms, which dynamically adjust $C$ based on statistical properties of the updates, are an emerging solution to balance this trade-off effectively.
\end{itemize}

In FL, multiple rounds of training amplify the total privacy loss. Privacy accountants, such as the moments accountant \cite{abadi2016deep} and the R\'enyi DP accountant \cite{wang2019subsampled}, are used to track and manage this cumulative loss, enabling a principled way to guarantee privacy over multiple iterations.

Selecting the appropriate clipping norm $C$ is one of the most significant challenges in FL with DP. The clipping threshold determines the trade-off between utility and privacy: a small $C$ effectively limits the sensitivity of the gradients but may clip valuable updates, reducing model accuracy. On the other hand, a large $C$ preserves more information but increases the noise required to satisfy privacy guarantees, thereby degrading utility.

\subsection{Related Work}

Prior research in DP-FL has explored various approaches to improve privacy-utility trade-offs. Early works such as DP-SGD \cite{abadi2016deep} introduced noise addition at the client level to ensure differential privacy. Another approach, DP-FedAvg \cite{mcmahan2017learning}, modifies the federated averaging procedure to incorporate privacy guarantees while maintaining communication efficiency.

Gradient clipping has been widely studied as a method to control sensitivity in DP-FL. Fixed clipping norms \cite{geyer2017differentially} provide stability but struggle to balance privacy and utility effectively. Recent works have explored adaptive clipping strategies \cite{andrew2021differentially}, where the clipping norm is dynamically adjusted based on gradient statistics. These techniques reduce unnecessary information loss while still ensuring privacy constraints.

Dynamic DP-SGD \cite{du2021dynamic} has emerged as a promising alternative, adjusting clipping thresholds and noise scales dynamically across training steps. This approach stabilizes updates and mitigates the performance degradation observed in traditional DP-SGD. The work in \cite{papernot2021tempered} further highlights how tempered sigmoid activations can be leveraged to implicitly control gradient norms, reducing the need for aggressive clipping and thereby improving accuracy under DP constraints.

Additionally, DP-SGD-WAV \cite{ranaweera2023improving} refines the application of DP in FL by incorporating a wavelet-based adaptive variance reduction mechanism, leading to improved model convergence and utility across various noise settings. These approaches collectively highlight the importance of adaptivity in addressing the trade-off between privacy and utility in FL .

\section{Adversarial Model}\label{sec:adv_model}

Our adversarial model assumes an honest-but-curious central server that faithfully executes the FL protocol but may attempt to infer sensitive information from the intermediate parameters shared by the clients. We also consider external adversaries that could intercept the communication between clients and the server, attempting to extract private client information.

The transmitted parameters in the FL process may expose client-specific information, such as rare features or unique patterns. An adversary could exploit these exposures to launch attacks. We focus on two primary privacy threats:

\begin{enumerate}
    \item \textbf{Membership Inference Attacks:} These attacks attempt to determine whether a particular client's data was used during model training. Such attacks can reveal sensitive information about a client, such as their medical records or financial activities. Adversaries achieve this by analyzing the intermediate parameters shared during aggregation \cite{hu2022membership}.
    
    \item \textbf{Model Inversion Attacks:} These attacks aim to reconstruct sensitive data from the aggregated parameters or the final global model. By leveraging patterns in the shared parameters, an adversary can recover private information, such as unique data points or identifiable features \cite{fredrikson2015model}.
\end{enumerate}

To counter these threats, we employ DP techniques at key stages of the FL process. By introducing carefully calibrated noise, we ensure that neither the intermediate parameters shared with the server nor the final global model reveal private client information. The inclusion of DP, however, introduces challenges such as reduced model accuracy due to noise injection.

Our methodology explores the sample-level DP approach where Noise is added during the stochastic gradient descent (SGD) training process at each client\cite{abadi2016deep}. This ensures that individual data samples remain indistinguishable within a client’s local dataset.

We assume that the adversary has full knowledge of the FL system, including its design, algorithms, and hyperparameters. This reflects a worst-case scenario where the adversary exploits all available information to compromise client privacy. Additionally, we acknowledge that while DP provides strong theoretical guarantees, its practical deployment must carefully balance privacy and utility, particularly in scenarios involving non-IID client data and limited participation.

Our proposed system framework demonstrates how FL can be augmented with robust privacy-preserving mechanisms to address contemporary privacy threats, ensuring secure and effective collaborative learning across distributed data sources.

\section{Methodology}\label{sec:method}

This section presents a MOO technique designed to dynamically adjust the clipping norm in FL. The proposed methodology integrates privacy and utility objectives into a unified framework, employing a sample-level DP approach to achieve an optimal balance between the two goals during training.

\subsection{Multi Objective Optimization for Adjusting the Clipping Norm}

In FL, the clipping norm is a pivotal parameter that determines the trade-off between privacy and utility. It governs the extent to which individual gradients are scaled before noise is added, directly influencing the level of DP achieved and the amount of useful information retained in the aggregated updates. To address this, we formulate a composite objective function that combines model utility and clipping norm regularization. The total loss function is defined as:

\begin{equation}
\label{eq:A1}
    \mathcal{L} = \mathcal{L}_{\text{model}} + \kappa \cdot \mathcal{L}_{\text{clipping}},
\end{equation}
where $\mathcal{L}_{\text{model}}$ represents the model loss, ensuring utility by minimizing the error in predictions. The term $\mathcal{L}_{\text{clipping}} = C$ penalizes large clipping norms, which directly influence the sensitivity of updates and the amount of noise required for DP. The regularization weight $\kappa$ controls the trade-off between privacy and utility, enabling fine-grained tuning of the optimization process. By minimizing this composite objective function, the methodology dynamically adapts the clipping norm to balance the competing objectives of privacy preservation and model accuracy.

The clipping norm $C$ is updated during training by calculating the gradient of the composite objective function with respect to $C$. This allows the optimization to respond to the evolving dynamics of the training process. Specifically, the gradient is computed as:
\begin{equation}
\label{eq:A2}
    \nabla_C \mathcal{L} = \kappa - \frac{\partial \mathcal{L}_{\text{model}}}{\partial C} \cdot \frac{1}{C},
\end{equation}
where the term $\frac{\partial \mathcal{L}_{\text{model}}}{\partial C}$ reflects the sensitivity of the model loss to changes in the clipping norm. The clipping norm is then updated using a gradient descent step:
\begin{equation}
\label{eq:A3}
    C \gets C - \eta_C \cdot \nabla_C \mathcal{L},
\end{equation}
where $\eta_C$ is the learning rate for the clipping norm. To ensure stability, the clipping norm is constrained to remain positive and within a reasonable range:
\begin{equation}
    C \gets \max(C, 10^{-3}).
\end{equation}
This ensures that the clipping norm does not become excessively small, which could destabilize training or compromise utility.

The methodology is grounded in a sample-level DP approach, where noise is added to individual client updates after clipping. By clipping gradients at a dynamically optimized norm, the sensitivity of updates is tightly controlled, enabling effective noise calibration. This combination ensures that the contributions of individual clients remain private while maintaining the overall utility of the federated model.

The dynamic adjustment process involves calculating the total loss at each epoch as the sum of the model loss and the clipping regularization term. The gradient of this loss with respect to the clipping norm is used to iteratively refine $C$, ensuring that the trade-off between privacy and utility is optimized throughout training. This integration of dynamic clipping and sample-level DP offers a principled approach to addressing the challenges of FL with DP.

The proposed technique is implemented as follows. First, the total loss for each epoch is computed by combining the model loss and the clipping regularization term. The gradient of the composite loss with respect to the clipping norm is then calculated, and the clipping norm is updated using a gradient descent step. Noise is added to the clipped gradients for DP before aggregating the updates to refine the model. The complete procedure is summarized in Algorithm~\ref{alg1}.

\begin{algorithm}
\label{alg1}
  \caption{Pseudocode of proposed DP-FL with MOO for adjusting the clipping norm}
  \begin{algorithmic}[1]
    \Require $\mathcal{S}_t$ subset of clients selected with a selection probability of $q_c \in (0, 1]$, client $k$'s local dataset, loss function $L(\theta, x_{i})$, learning rate $\eta_t$, learning rate for clipping norm $\eta_c$
    \Ensure Global model update $\theta_{t+1}$

    \State \textbf{Server executes:}
    \State Initialize global model $\theta_{t}$ randomly
    \For{$t \in \{1, 2, ..., T\}$}
      \For{each client $k \in \mathcal{S}_t$}
        \State $\theta_{t+1}^k \gets \text{ClientTraining}(\theta_{t}, k)$
        
      \EndFor

      \State \textbf{Model aggregation:}
      \State $\theta_{t+1} \gets \sum_{k \in \mathcal{S}_t} \frac{n_k}{n_{\mathcal{S}_t}} \theta_{t+1}^k$
    \EndFor
    
    \hrulefill
    \State \textbf{ClientTraining}($\theta_{t}, k$):
    \State \hspace{1.5em} \textbf{for} $n \in \{1, 2, ..., N\}$ \textbf{do}
      \State \hspace{3em} Sample local batch $L_n^k$ with probability $q$
      \State \hspace{3em} \textbf{for} each sample $x_i^k \in L_n^k$ \textbf{do}
        \State \hspace{4.5em}Compute gradient of $\mathcal{L}$ w.r.t $C$: 
         \State \hspace{5.5em}$\nabla_C \mathcal{L} \gets \nabla_{C}\mathcal{L}(\theta_{t}, x_{i}^k)$
        \State \hspace{4.5em}Update clipping norm: 
         \State \hspace{5.5em}$C \gets \max(C - \eta_C \cdot \nabla_C \mathcal{L}, 10^{-3})$
         \State \hspace{4.5em}Compute gradient of $\mathcal{L}$ w.r.t $\theta_t$: 
         \State \hspace{5.5em}$g_t^k(x_i) \gets \nabla_{\theta_t}\mathcal{L}(\theta_{t}, x_{i}^k)$
        \State \hspace{4.5em}Clip gradients:  \State \hspace{5.5em}$\bar{g}_t^k(x_i) \gets g_t^k(x_i)/ \text{max} (1, \dfrac{||g_t^k(x_i)||_2}{C})$
        \State \hspace{4.5em} Add noise for DP: 
         \State \hspace{5.5em}$\tilde{g}_t^k(x_i) \gets \dfrac{1}{L_n^k}(\sum_i \bar{g}_t^k(x_i) + \mathcal{N}(0, \sigma^2C^2I))$
        \State \hspace{4.5em}Update model: 
        \State \hspace{5.5em} $\theta_{t+1}^k \gets \theta_t - \eta_t\tilde{g}_t^k(x_i)$
        
      \State \hspace{3em} \textbf{end for}
    \State \hspace{1.5em} \textbf{end for}
    \State \hspace{1.5em} \textbf{Return $\theta_{t+1}^k$}

  \end{algorithmic}
\end{algorithm}

By integrating sample-level DP with dynamic clipping norm optimization, this methodology ensures a flexible and effective means to manage the trade-off between privacy and utility in FL. The overall workflow of the proposed DP-FL framework with multi-objective optimization for adaptive clipping norm adjustment is illustrated in Fig.~\ref{fig:moo_overview}.

\begin{figure*}[!t]
    \centering
    \includegraphics[width=\linewidth]{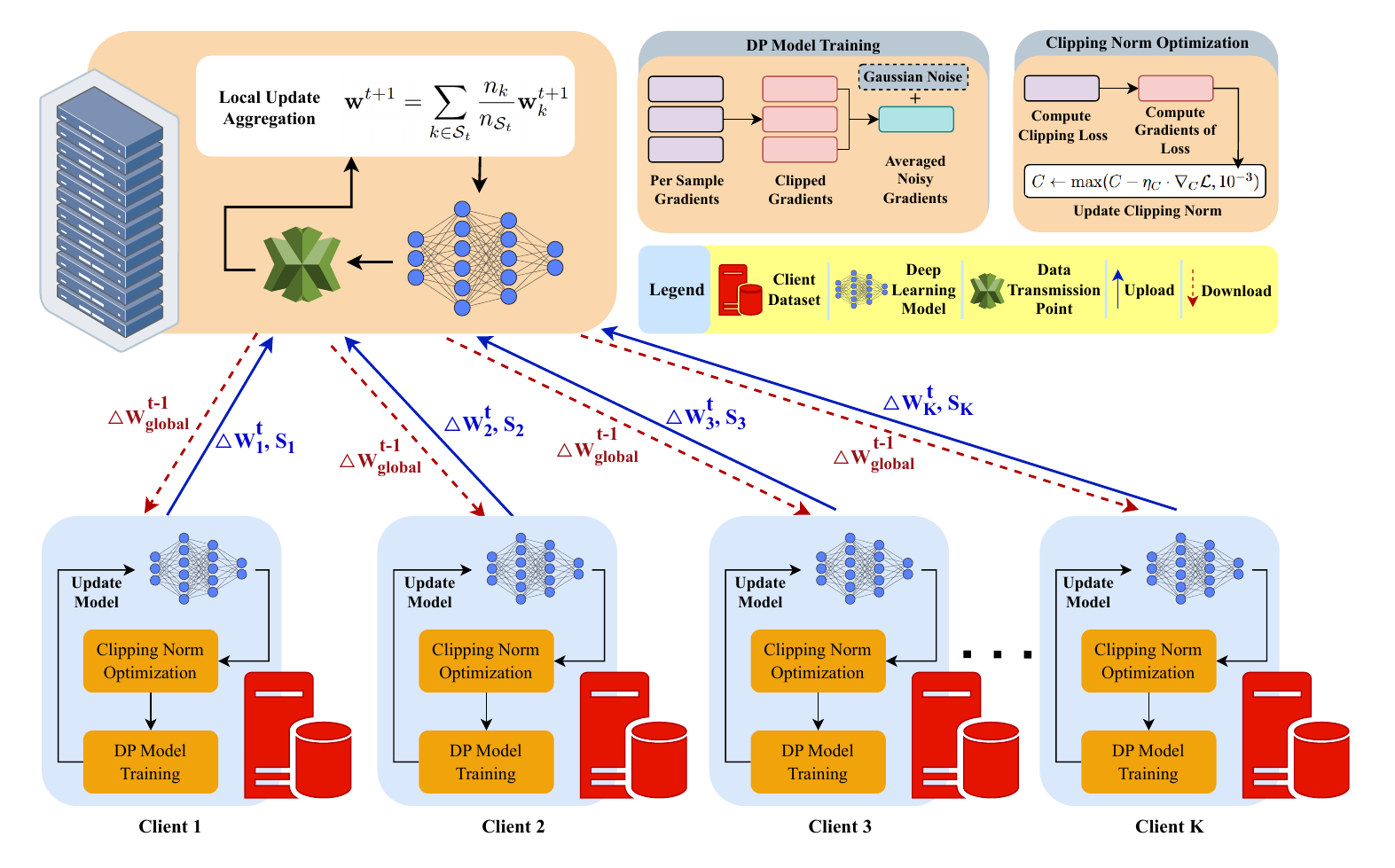}
    \vspace{-3em}
    \caption{An overview of the proposed DP-FL framework with multi-objective optimization (MOO) for adaptive clipping norm adjustment.}
    \label{fig:moo_overview}
\end{figure*}

\subsection{Convexity Analysis of the Optimization Problem}

In this section, we analyze the convexity of the MOO problem defined in \ref{eq:A1}. Given that $\mathcal{L}_{clipping}=C$, we can reformulate this as:
\begin{equation}
    \mathcal{L}(C) = \mathcal{L}_{\text{model}}(C) + \kappa \cdot C,
\end{equation}
where $\mathcal{L}_{\text{model}}(C)$ is the model loss term dependent on the clipping norm $C$, and $\kappa \cdot C$ represents the linear regularization term. Here, $\kappa > 0$ is a regularization parameter, and $C > 0$ denotes the clipping norm. To determine whether this optimization problem is well-posed and possesses a global minimum, we conduct a detailed convexity analysis under the Polyak-\L{}ojasiewicz (PL) condition.

A function $f(x)$ satisfies the PL condition if there exists a constant $\mu > 0$ such that:
\begin{equation}
    \frac{1}{2} \|\nabla f(x)\|^2 \geq \mu (f(x) - f^*),
\end{equation}
where $f^*$ is the optimal function value. Unlike strong convexity, which directly imposes curvature constraints, the PL condition ensures that gradient norm decay implies convergence to the optimal solution. This property is sufficient to guarantee global convergence in many machine learning problems.

The regularization term $\kappa \cdot C$ is linear in $C$, and its second derivative is
\begin{equation}
    \frac{d^2}{dC^2} (\kappa \cdot C) = 0.
\end{equation}
Thus, it is trivially convex.

For the model loss term, we use the PL condition to show convergence behavior. Differentiating $\mathcal{L}_{\text{model}}(C)$, we obtain:
\begin{equation}
    \mathcal{L}_{\text{model}}(C) - \mathcal{L}_{\text{model}}(C^*) \leq \frac{1}{2\mu} \|\nabla \mathcal{L}_{\text{model}}(C)\|^2.
\end{equation}
Applying the Lipschitz continuity assumption:
\begin{equation}
    \|\nabla \mathcal{L}_{\text{model}}(C) - \nabla \mathcal{L}_{\text{model}}(C')\| \leq M |C - C'|,
\end{equation}
we conclude that the loss function is smooth. Combining the PL condition and smoothness, we derive that $\mathcal{L}_{\text{model}}(C)$ exhibits a form of pseudo-convexity, ensuring that gradient-based updates lead to optimal convergence.

Since the composite function $\mathcal{L}(C)$ consists of a convex regularization term and a PL-satisfying loss term, its convergence is ensured. The first-order optimality condition is given by:
\begin{equation}
    \frac{d \mathcal{L}}{dC} = \kappa - \frac{\partial \mathcal{L}_{\text{model}}}{\partial C} \cdot \frac{1}{C}.
\end{equation}
Setting $\frac{d \mathcal{L}}{dC} = 0$, we obtain:
\begin{equation}
    \kappa \cdot C = \frac{\partial \mathcal{L}_{\text{model}}}{\partial C}.
\end{equation}
Solving this equation numerically yields the optimal clipping norm $C^*$, and the PL condition guarantees convergence to this minimum.

By combining the convexity of the regularization term and the PL-based convergence properties of the model loss term, we conclude that the composite loss function $\mathcal{L}(C)$ is well-posed and possesses a unique minimizer. This analysis ensures that the optimization problem is solvable and that gradient-based methods will efficiently converge to the optimal clipping norm.

\subsection{Convergence Analysis}

In this section, we analyze the convergence performance of our proposed FL algorithm with adaptive clipping and sample-level DP. Our objective is to establish how different factors, including the choice of the clipping norm $C$, affect convergence, particularly due to its direct influence on the noise variance $\sigma^2$ introduced for DP. We begin by making the following \textbf{assumptions:}

\begin{enumerate} 
    \item \textbf{Bounded Gradient Dissimilarity:} There exist constants $B_1$ and $B_2$ such that:
    \begin{equation}
        \sum_{k=1}^{K} \|\nabla L_k(\theta) - \nabla L(\theta) \|^2 \leq B_1 \|\nabla L(\theta)\|^2 + \frac{B_2^2}{K}.
    \end{equation}
    
    \item \textbf{Lipschitz Continuity:} The local loss function $L_k(\theta)$ is $M$-Lipschitz continuous, meaning that there exists a constant $M$ such that:
    \begin{equation}
        \|\nabla L_k(\theta) - \nabla L_k(\theta')\| \leq M \|\theta - \theta'\|, \quad \forall \theta, \theta'.
    \end{equation}
    
    \item \textbf{Polyak-Łojasiewicz (PL) Inequality:} The global loss function $L(\theta)$ satisfies the PL inequality, i.e., there exists a positive scalar $\mu > 0$ such that:
    \begin{equation}
        \frac{1}{2} \|\nabla L(\theta)\|^2 \geq \mu (L(\theta) - L(\theta^*)),
    \end{equation}
    where $\theta^*$ is the minimizer of $L(\theta)$.
\end{enumerate}

The first assumption, bounded gradient dissimilarity, ensures that the variance in client gradients is controlled, limiting the deviation of local model updates from the global direction. This is critical for stable aggregation and prevents excessive divergence in heterogeneous settings. The second assumption, Lipschitz continuity, guarantees that the gradient changes smoothly, preventing abrupt shifts in the optimization landscape and ensuring that updates remain predictable. The third assumption, the PL condition, provides a relaxation of strong convexity, ensuring that the gradient norm remains a valid indicator of progress toward convergence, even in non-convex settings.

\begin{lemma}
\label{thm:convergence_FL_DP}
Consider the sequence of model parameters ${\theta_t}$ where $t \geq 0$, generated by the proposed FL algorithm (Algorithm 1). Assume that each local loss function $L_k(\theta)$ satisfies the Lipschitz continuity condition and that the gradient dissimilarity is bounded. Under these assumptions, the expected difference between the global loss function $L(\theta)$ at iteration $t+1$ and the optimal loss $L(\theta^*)$ is bounded by the following expression:

\begin{equation}
    \begin{aligned}
        & \mathbb{E}[L(\theta_{t+1})] - L(\theta^*) 
        \leq \triangle_t \mathbb{E}[L(\theta_t) - L(\theta^*)] + c_t +\\& \frac{\eta_t}{2} \left[-1 + \lambda M \eta_t \left(\frac{B_1 + K}{K} \right) \right] 
        \left\| \sum_{k=1}^{K} \frac{n_k}{n_{\mathcal{S}_t}} \nabla_{\theta_{t}^k} \widehat{L}_k(\theta_t, x_i^k) \right\|^2 \\
        &\quad + B_t \sum_{j=t_c+1}^{t-1} \eta_j^2 
        \left\| \sum_{k=1}^{K} \frac{n_k}{n_{\mathcal{S}_t}} \nabla_{\theta_{t}^k} \widehat{L}_k(\theta_t, x_i^k) \right\|^2.
    \end{aligned}
\end{equation}

where the parameters are defined as follows:

\begin{equation}
    \triangle_t = 1 - \mu \eta_t,
\end{equation}
\begin{equation}
    c_t = \frac{\eta_t M B_2^2}{K} \left[\frac{\eta_t}{2} + \frac{M(K+1)}{K} \sum_{j=t_c+1}^{t-1} \eta_j^2 \right],
\end{equation}
\begin{equation}
    B_t = \frac{\lambda (K+1) \eta_t M^2}{K^2} (B_1 + N),
\end{equation}

\begin{equation}
\begin{aligned}
    \nabla_{\theta_{t}^k} \widehat{L}_k(\theta_{t}, x_{i}^k) = \nabla_{\theta_{t}^k} L_k(\theta_{t}, x_{i}^k) + \mathcal{N}\left(0, \sigma^2 \right).
\end{aligned}
\end{equation}
\end{lemma}

\begin{proof}
See Appendix~\ref{sec:app1}.
\end{proof}

One key observation from Theorem \ref{thm:convergence_FL_DP} is the role of the learning rate $\eta_t$ in determining the convergence speed. The condition $\triangle_t = 1 - \mu \eta_t < 1$ ensures that the expected loss decreases over time, provided that $\eta_t$ is chosen appropriately relative to the strong convexity parameter $\mu$. A large $\eta_t$ accelerates convergence but may introduce instability, while a small $\eta_t$ slows progress. Additionally, the presence of Gaussian noise, parameterized by $\sigma^2$, influences the convergence rate. Since $\sigma^2$ is proportional to the clipping norm $C$, larger values of $C$ lead to higher noise levels, which can slow down convergence. Conversely, smaller values of $C$ reduce the noise but may lead to excessive gradient clipping, adversely impacting utility. Thus, the optimal selection of $C$ is critical in balancing privacy and model accuracy.

\begin{lemma}
\label{lemma:multi_obj_convergence}
Consider the MOO problem defined in \ref{eq:A1}
where $\mathcal{L}_{\text{model}}(C)$ satisfies the PL condition and is $L$-smooth. Assume that the gradient of the model loss satisfies bounded variance, 
Under an appropriate step size $\eta_c$, the iterates of the clipping norm update rule in \ref{eq:A3} converge linearly to an optimal clipping norm $C^*$, satisfying:
\begin{equation}
    \mathbb{E}[\mathcal{L}(C_t) - \mathcal{L}(C^*)] \leq (1 - 2 \mu \eta_c)^t (\mathcal{L}(C_0) - \mathcal{L}(C^*)) + \frac{\eta_c \sigma_g^2}{2 \mu},
\end{equation}
where $0 < \eta_c \leq \frac{2}{L}$ ensures stability.
\end{lemma}

\begin{proof}
Using the $L$-smooth property of $\mathcal{L}(C)$, we have:
\begin{equation}
    \mathcal{L}(C_{t+1}) \leq \mathcal{L}(C_t) + \nabla \mathcal{L}(C_t) (C_{t+1} - C_t) + \frac{L}{2} (C_{t+1} - C_t)^2.
\end{equation}
Substituting $C_{t+1} - C_t = -\eta_c \nabla \mathcal{L}(C_t)$ gives:
\begin{equation}
    \mathcal{L}(C_{t+1}) \leq \mathcal{L}(C_t) - \eta_c \|\nabla \mathcal{L}(C_t)\|^2 + \frac{L \eta_c^2}{2} \|\nabla \mathcal{L}(C_t)\|^2.
\end{equation}
Applying the PL condition:
\begin{equation}
    \frac{1}{2} \|\nabla \mathcal{L}(C_t)\|^2 \geq \mu (\mathcal{L}(C_t) - \mathcal{L}(C^*)).
\end{equation}
Rearranging terms and incorporating gradient noise variance $\sigma_g^2$, we obtain the geometric convergence bound with an error floor due to variance:
\begin{equation}
    \mathbb{E}[\mathcal{L}(C_t) - \mathcal{L}(C^*)] \leq (1 - 2 \mu \eta_c)^t (\mathcal{L}(C_0) - \mathcal{L}(C^*)) + \frac{\eta_c \sigma_g^2}{2 \mu}.
\end{equation}
This result shows that while the iterates of $C$ converge, the final error is influenced by gradient variance, which is in turn affected by the choice of the clipping norm.
\end{proof}

This lemma establishes that the adaptive clipping norm optimization converges linearly, ensuring stability and effectiveness in balancing privacy and utility. By dynamically adjusting $C$, we can mitigate the negative impact of DP noise on model accuracy while maintaining privacy guarantees. The combination of Theorem \ref{thm:convergence_FL_DP} and Lemma \ref{lemma:multi_obj_convergence} demonstrates that our proposed technique achieves both DP and efficient model training without excessive degradation in performance.

\section{Results and Discussion}\label{sec:results}

\subsection{Experimental Setup}
Our experiments were implemented in Python using the PyTorch framework, leveraging an Intel Core i7 11th generation processor alongside an NVIDIA RTX 3080 GPU for accelerated computations. To monitor the accumulated privacy loss, we employed the Rényi Differential Privacy (RDP) accountant provided by Google’s DP library.

To evaluate our adaptive clipping mechanism, we conducted experiments on three benchmark datasets: MNIST\cite{lecun1998mnist}, Fashion-MNIST\cite{https://doi.org/10.48550/arxiv.1708.07747}, and CIFAR-10\cite{Krizhevsky09learningmultiple}. To train on these datasets, we employ convolutional neural network (CNN) architectures designed to handle the complexity of each dataset efficiently. The architectural details are presented in Tables \ref{tab:mnist_fmnist} and \ref{tab:cifar10}.

To ensure fair performance evaluation, our method's starting clipping norm was set to match the static clipping norms used in FL w/ DP-SGD, FL w/ DP-FedAvg, FL w/ DP-SGD-WAV, and DP-FL w/ Tempered Sigmoid, as well as the starting clipping norm of FL w/ Dynamic DP. While some of these baseline methods use fixed clipping norms and others employ adaptive strategies, our approach uniquely integrates a multi-objective optimization framework that dynamically adjusts the clipping bound based on training dynamics, ensuring an optimal balance between privacy and model utility. This ensures that our comparisons isolate the effects of adaptive clipping rather than differences in initialization.

A comprehensive hyperparameter tuning process was conducted to optimize the performance of our proposed approach while maintaining privacy guarantees. Various hyperparameters, including learning rate, batch size, weight decay, and network depth, were explored using a combination of domain expertise and automated search methods such as grid search. The privacy parameter $\epsilon$ was systematically varied to assess the trade-off between privacy preservation and model accuracy, ensuring a balanced approach.

Early stopping was not applied to any method during training to ensure a fair comparison of convergence behavior. The number of communication rounds was fixed across all methods to 200 to eliminate any bias introduced by training duration. These configurations ensure that the reported improvements stem from the adaptive clipping mechanism rather than variations in stopping criteria.

\begin{table}[t!]
\centering
\caption{CNN model architecture used for MNIST and Fashion-MNIST datasets.}
\label{tab:mnist_fmnist}
\begin{tabular}{ccc}
\hline
\textbf{Layer} & \textbf{Type} & \textbf{Parameters} \\
\hline \hline
1 & Input & 28x28 Grayscale Image \\
2 & 2D Convolution & 16 filters, 8x8 kernel, stride 2, ReLU \\
3 & 2D Max-Pooling & 2x2 kernel \\
4 & 2D Convolution & 32 filters, 4x4 kernel, stride 2, ReLU \\
5 & 2D Max-Pooling & 2x2 kernel \\
6 & Fully Connected & 32 units, ReLU \\
7 & Output & 10 units, Softmax \\
\hline
\end{tabular}
\end{table}

\begin{table}[t!]
\centering
\caption{CNN model architecture used for CIFAR-10 dataset.}
\label{tab:cifar10}
\begin{tabular}{ccc}
\hline
\textbf{Layer} & \textbf{Type} & \textbf{Parameters} \\
\hline \hline
1  & Input  & 32x32x3 RGB Image \\
2  & 2D Convolution  & 32 filters, 3x3 kernel, stride 1, ReLU \\
3  & 2D Average-Pooling & 2x2 kernel, stride 2 \\
4  & 2D Convolution  & 64 filters, 3x3 kernel, stride 1, ReLU \\
5  & 2D Average-Pooling & 2x2 kernel, stride 2 \\
6  & 2D Convolution  & 64 filters, 3x3 kernel, stride 1, ReLU \\
7  & 2D Average-Pooling & 2x2 kernel, stride 2 \\
8  & 2D Convolution  & 128 filters, 3x3 kernel, stride 1, ReLU \\
9  & \begin{tabular}[c]{@{}c@{}}2D Adaptive\\ Average-Pooling\end{tabular} & - \\
10 & Output & 10 units, Softmax \\
\hline
\end{tabular}
\end{table}

\subsection{Datasets}
The evaluation of our proposed method was conducted using three widely used benchmark datasets for image classification tasks: MNIST, Fashion-MNIST (FMNIST), and CIFAR-10. Each dataset presents unique challenges, allowing us to assess the adaptability and effectiveness of our approach across different data distributions and complexities.

\textbf{MNIST:} The MNIST dataset consists of 70,000 grayscale images of handwritten digits, ranging from 0 to 9. Each image has a resolution of $28\times28$ pixels, with 60,000 samples used for training and 10,000 for testing. As a simple yet fundamental dataset in machine learning research, MNIST serves as a baseline for evaluating the effectiveness of different algorithms in recognizing handwritten characters.

\textbf{Fashion-MNIST (FMNIST):} FMNIST is a drop-in replacement for MNIST, consisting of 70,000 grayscale images of various clothing items, such as shirts, trousers, and shoes, categorized into 10 distinct classes. Like MNIST, each image has a resolution of $28\times28$ pixels, with the same 60,000/10,000 split for training and testing. This dataset is more complex than MNIST due to increased intra-class variations, making it a more challenging benchmark for evaluating classification models.

\textbf{CIFAR-10:} The CIFAR-10 dataset is a collection of 60,000 color images, each of size $32\times32$ pixels, spanning 10 object categories, including animals and vehicles. Unlike MNIST and FMNIST, CIFAR-10 presents significantly more visual complexity, requiring models to learn rich feature representations. The dataset is split into 50,000 training images and 10,000 test images. Given its higher-dimensional color images and diverse classes, CIFAR-10 is often used to assess the scalability and robustness of deep learning models.

Each of these datasets was partitioned among FL clients to simulate realistic decentralized learning environments. The non-iid (non-independent and identically distributed) nature of the data distributions across clients further challenged the learning process, making it an ideal setting to evaluate the efficacy of our adaptive clipping mechanism.
\begin{table*}[t!]
\centering
\caption{Classification accuracy (\%) of different FL methods under varying privacy budgets ($\varepsilon$) across MNIST, Fashion-MNIST, and CIFAR-10 datasets. The privacy parameter 
$\delta$ is fixed at $10^{-5}$ for all experiments. }
\label{tab:results}
\begin{tabular}{@{}cccccccccc@{}}
\toprule
\multirow{2}{*}{\textbf{Method}} & \multicolumn{3}{c}{\textbf{MNIST}}                                                               & \multicolumn{3}{c}{\textbf{Fashion MNIST}}                                                       & \multicolumn{3}{c}{\textbf{CIFAR10}}                                                             \\ \cmidrule(l){2-10} 
                                 & \multicolumn{3}{c}{\textbf{\begin{tabular}[c]{@{}c@{}}Classification\\ (Accuracy)\end{tabular}}} & \multicolumn{3}{c}{\textbf{\begin{tabular}[c]{@{}c@{}}Classification\\ (Accuracy)\end{tabular}}} & \multicolumn{3}{c}{\textbf{\begin{tabular}[c]{@{}c@{}}Classification\\ (Accuracy)\end{tabular}}} \\ \midrule
Non-Private FL                   & \multicolumn{3}{c}{96.25\%}                                                                      & \multicolumn{3}{c}{89.45\%}                                                                      & \multicolumn{3}{c}{72.61\%}                                                                      \\ \cmidrule(l){2-10} 
                                 & $\varepsilon$=6.38                    & $\varepsilon$=3.61                    & $\varepsilon$=1.64                    & $\varepsilon$=6.38                    & $\varepsilon$=3.61                    & $\varepsilon$=1.64                    & $\varepsilon$=6.38                    & $\varepsilon$=3.61                    & $\varepsilon$=1.64                    \\ \cmidrule(l){2-10} 
FL w/ DP-SGD                     & 86.39\%                        & 81.58\%                        & 76.43\%                        & 77.19\%                        & 69.35\%                        & 58.14\%                        & 60.53\%                        & 53.68\%                        & 48.88\%                        \\
FL w/ DP-FedAvg                  & 88.42\%                        & 82.50\%                        & 76.58\%                        & 75.18\%                        & 68.91\%                        & 53.58\%                        & 59.17\%                        & 51.85\%                        & 43.06\%                        \\
FL w/ Dynamic DP                 & 87.71\%                        & 86.59\%                        & 82.81\%                        & 79.62\%                        & 78.44\%                        & 69.82\%                        & 62.61\%                        & 59.02\%                        & 53.23\%                        \\
FL w/ DP-SGD-WAV                 & 90.24\%                        & 89.72\%                        & 85.61\%                        & 83.53\%                        & 82.16\%                        & 72.37\%                        & 64.58\%                        & 61.18\%                        & 57.85\%                        \\
DP-FL w/ Tempered Sigmoid        & 90.92\%                        & 90.02\%                        & 84.39\%                        & 82.73\%                        & 81.49\%                        & 70.58\%                        & 64.15\%                        & 60.59\%                        & 56.02\%                        \\

DP-FL w/ MOO                     & 91.10\%                        & 90.21\%                        & 88.16\%                        & 84.20\%                        & 83.27\%                        & 74.78\%                        & 64.37\%                        & 62.92\%                        & 59.89\%                        \\ \bottomrule
\end{tabular}
\end{table*}

\subsection{Discussion}

Table \ref{tab:results} summarizes the classification accuracy of different FL methods under varying privacy budgets ($\epsilon$ values). The results highlight the trade-off between privacy and model accuracy, demonstrating the effectiveness of our proposed method. Our approach, DP-FL with MOO, consistently outperforms the highest-performing state-of-the-art method, whether it is DP-FL with tempered sigmoid or FL with DP-SGD-WAV, across most datasets and privacy levels. For MNIST, DP-FL with MOO achieves an accuracy improvement of 0.18\% at $\epsilon=6.38$, 0.19\% at $\epsilon=3.61$, and 2.55\% at $\epsilon=1.64$ compared to the best competing method. In Fashion-MNIST, our method surpasses the best alternative by 0.67\% at $\epsilon=6.38$, 1.11\% at $\epsilon=3.61$, and 2.41\% at $\epsilon=1.64$. The results on CIFAR-10 further demonstrate our method’s advantage, outperforming the best competing technique by 1.74\% at $\epsilon=3.61$ and 2.04\% at $\epsilon=1.64$. However, at $\epsilon=6.38$ for CIFAR-10, DP-FL with MOO falls slightly short, achieving 0.21\% lower accuracy than FL with DP-SGD-WAV, which performs best in this setting. 

These improvements indicate that our adaptive clipping mechanism enhances model utility under strict privacy constraints, ensuring improved classification accuracy across different datasets. As expected, lower values of $\epsilon$ (stronger privacy guarantees) result in decreased accuracy due to increased noise injection for DP. However, our method mitigates this degradation more effectively than existing techniques by dynamically optimizing the clipping norm. Compared to fixed-clipping methods, our adaptive strategy prevents over-clipping, thereby preserving useful gradient information. This leads to more stable training and better generalization, particularly for high-dimensional datasets like CIFAR-10. The findings validate the effectiveness of our MOO approach in FL, demonstrating superior performance over existing techniques while maintaining strong privacy guarantees.

\section{Conclusion}\label{sec:conclusion}

In this work, we proposed an adaptive clipping mechanism for DP-FL that optimally balances the trade-off between privacy and model utility. Our approach leverages a MOO framework to dynamically adjust the clipping norm during training, mitigating the negative impact of excessive noise injection while ensuring rigorous privacy guarantees. Theoretical analysis established the convergence properties of our method, demonstrating its effectiveness in maintaining stable model updates. Extensive experiments conducted on MNIST, Fashion-MNIST, and CIFAR-10 showed that our method consistently outperforms existing state-of-the-art approaches across different privacy budgets, with particularly strong improvements in lower privacy regimes. While our approach performs well in most settings, future work could explore more advanced adaptive strategies to further refine the clipping norm selection and improve performance on complex datasets such as CIFAR-10 under high privacy constraints. Overall, our results demonstrate the potential of adaptive clipping in enabling privacy-preserving federated learning without significant loss in model performance.

\bibliographystyle{IEEEtran}
\bibliography{main}

\newpage
\appendix

\subsection{Proof of Lemma~\ref{thm:convergence_FL_DP} }
\label{sec:app1}
\textbf{Lemma: } Consider the sequence of model parameters ${\theta_t}_{t \geq 0}$ generated by the proposed FL algorithm (Algorithm \ref{alg1}). Assume that each local loss function $L_k(\theta)$ satisfies the Lipschitz continuity condition and that the gradient dissimilarity is bounded. Under these assumptions, the expected difference between the global loss function $L(\theta)$ at iteration $t+1$ and the optimal loss $L(\theta^*)$ is bounded by the following expression:

\begin{equation}
    \begin{aligned}
        \mathbb{E}&[L(\theta_{t+1})] - L(\theta^*) 
        \leq \triangle_t \mathbb{E}[L(\theta_t) - L(\theta^*)] + c_t +\\
        &\frac{\eta_t}{2} \left[-1 + \lambda M \eta_t \left(\frac{B_1 + K}{K} \right) \right] 
        \left\| \sum_{k=1}^{K} \frac{n_k}{n_{\mathcal{S}_t}} \nabla_{\theta_{t}^k} \widehat{L}_k(\theta_t, x_i^k) \right\|^2 \\
        &\quad + B_t \sum_{j=t_c+1}^{t-1} \eta_j^2 
        \left\| \sum_{k=1}^{K} \frac{n_k}{n_{\mathcal{S}_t}} \nabla_{\theta_{t}^k} \widehat{L}_k(\theta_t, x_i^k) \right\|^2.
    \end{aligned}
\end{equation}

where the parameters are defined as follows:

\begin{equation}
    \triangle_t = 1 - \mu \eta_t,
\end{equation}
\begin{equation}
    c_t = \frac{\eta_t M B_2^2}{K} \left[\frac{\eta_t}{2} + \frac{M(K+1)}{K} \sum_{j=t_c+1}^{t-1} \eta_j^2 \right],
\end{equation}
\begin{equation}
    B_t = \frac{\lambda (K+1) \eta_t M^2}{K^2} (B_1 + N),
\end{equation}

\begin{equation}
\begin{aligned}
    \nabla_{\theta_{t}^k} \widehat{L}_k(\theta_{t}, x_{i}^k) = \nabla_{\theta_{t}^k} L_k(\theta_{t}, x_{i}^k) + \mathcal{N}\left(0, \sigma^2 \right).
\end{aligned}
\end{equation}

\begin{proof}

From the local update step in Algorithm~\ref{alg1}, we can express the parameter update as:
\begin{equation}
    \theta_{t+1}^k = \theta_t^k - \eta_t g^k_t(x_{i}^k),
\end{equation}
where the gradient estimate is given by:
\begin{equation}
    g^k_t(x_{i}^k) = \nabla_{\theta_t^k} L_k(\theta_t, x_{i}^k) + \mathcal{N}(0, \sigma^2).
\end{equation}

The global model update is performed as:
\begin{equation}
    \theta_{t+1} = \sum_{k=1}^{K} \frac{n_k}{n_{\mathcal{S}_t}} \theta_{t+1}^k.
\end{equation}

Now, define:
\begin{equation}
    \widehat{g}_t = \sum_{k=1}^{K} \frac{n_k}{n_{\mathcal{S}_t}} g^k_t(x_{i}^k).
\end{equation}

Since we assume that the loss function $L_k(\theta)$ is Lipschitz continuous, we can bound the difference in loss between consecutive iterations:
\begin{equation}
    L(\theta_{t+1}) - L(\theta_t) \leq -\eta_t \langle \nabla L(\theta_t), \widehat{g}_t \rangle + \frac{\eta_t^2 M}{2} \|\widehat{g}_t\|^2.
\end{equation}

By taking the expectation over all participating devices, we obtain:
\begin{equation}
\begin{aligned}
\label{eq:100}
    \mathbb{E} \left[\mathbb{E}_{k \in \mathcal{S}_t} [L(\theta_{t+1}) - L(\theta_t)] \right] 
    &\leq -\eta_t \mathbb{E} \left[ \mathbb{E}_{k \in \mathcal{S}_t} \langle \nabla L(\theta_t), \widehat{g}_t \rangle \right] \\
    &\quad + \frac{\eta_t^2 M}{2} \mathbb{E} \left[ \mathbb{E}_{k \in \mathcal{S}_t} \|\widehat{g}_t\|^2 \right].
\end{aligned}
\end{equation}

Now, let $g_t$ denote the full gradient of the local objective function at iteration $t$, while $\widehat{g}_t$ is an unbiased estimator of $g_t$. Since mini-batches are selected in an i.i.d. manner across devices, we have:

\begin{equation}
\begin{aligned}
    &\mathbb{E} \left[ \|\widehat{g}_t - g_t\|^2 \right] 
    = \mathbb{E} \left[ \left\| \frac{1}{g_t^k} \mathbb{E}_{k \in \mathcal{S}_t} \widehat{g}_t^k - \frac{1}{g_t^k} \mathbb{E}_{k \in \mathcal{S}_t} g_t^k \right\|^2 \right] \\
    &= \frac{1}{K^2} \mathbb{E} \left[ \mathbb{E}_{k \in \mathcal{S}_t} \|\widehat{g}_t^k - g_t^k\|^2 + \sum_{i \neq k} \langle \widehat{g}_t^i - g_t^i, \widehat{g}_t^k - g_t^k \rangle \right] \\
    &= \frac{1}{K^2} \mathbb{E}_{k \in \mathcal{S}_t} \mathbb{E} \left[ \|\widehat{g}_t^k - g_t^k\|^2 \right] \\&\quad +\frac{1}{K^2} \sum_{i \neq k} \mathbb{E} \left[ \langle \widehat{g}_t^k - g_t^k, \widehat{g}_t^i - g_t^i \rangle \right] \\
    &\leq \frac{1}{K^2} \mathbb{E}_{k \in \mathcal{S}_t} \mathbb{E} \left[ \|\widehat{g}_t^k - g_t^k\|^2 \right] 
    \\& \quad + \frac{1}{K^2} \sum_{i \neq k} \langle \mathbb{E} \left[ \widehat{g}_t^k - g_t^k \right], \mathbb{E} \left[ \widehat{g}_t^i - g_t^i \right] \rangle \\
    &\leq \frac{1}{K^2} \mathbb{E}_{k \in \mathcal{S}_t} \left[ B_1 \|g_t^k\|^2 + B_2^2 \right] \\
    &= \frac{B_1}{K^2} \mathbb{E}_{k \in \mathcal{S}_t} \| g_t^k \|^2 + \frac{B_2^2}{g_t^k}.
\end{aligned}
\end{equation}

Next, by taking the expectation over the random selection of participating devices on both sides of the previous equation, we derive:
\begin{equation}
\begin{aligned}
    \mathbb{E}_{k \in \mathcal{S}_t} \left[ \mathbb{E} \left[ \|\widehat{g}_t - g_t\|^2 \right] \right] 
    &\leq \mathbb{E}_{k \in \mathcal{S}_t} \left[ \frac{B_1}{K^2} \mathbb{E}_{k \in \mathcal{S}_t} \| g_t^k \|^2 + \frac{B_2^2}{g_t^k} \right] \\
    &= \frac{B_1}{K^2} \mathbb{E}_{k \in \mathcal{S}_t} \left[ \mathbb{E}_{k \in \mathcal{S}_t} \| g_t^k \|^2 \right] + \frac{B_2^2}{g_t^k} \\
    &= \frac{B_1}{K^2} K \sum_{k=1}^{K} \frac{n_k}{n_{\mathcal{S}_t}} \| g_t^k \|^2 + \frac{B_2^2}{g_t^k}.
\end{aligned}
\end{equation}

Since we know that the expected value of $\widehat{g}_t^k$ equals $g_t^k$, i.e., $\mathbb{E}[\widehat{g}_t^k] = g_t^k$, we can express:
\begin{equation}
\begin{aligned}
    \mathbb{E}[\|\widehat{g}_t\|^2] 
    &= \mathbb{E}[\|\widehat{g}_t - \mathbb{E}[\widehat{g}_t]\|^2] + \|\mathbb{E}[\widehat{g}_t]\|^2 \\
    &= \mathbb{E}[\|\widehat{g}_t - g_t\|^2] + \|g_t\|^2 \\
    &\leq \frac{B_1}{K^2} \mathbb{E}_{k \in \mathcal{S}_t} \|g_t^k\|^2 + \frac{B_2^2}{K} + \left\|\frac{1}{K} \mathbb{E}_{k \in \mathcal{S}_t} g_t^k\right\|^2 \\
    &\leq \frac{B_1}{K^2} \mathbb{E}_{k \in \mathcal{S}_t} \|g_t^k\|^2 + \frac{B_2^2}{K} + \frac{1}{K} \mathbb{E}_{k \in \mathcal{S}_t} \|g_t^k\|^2 \\
    &= \left(\frac{B_1 + K}{K^2}\right) \mathbb{E}_{k \in \mathcal{S}_t} \|g_t^k\|^2 + \frac{B_2^2}{K}.
\end{aligned}
\end{equation}

Here, we use the fact that:
\begin{equation}
    \left\|\sum_{i=1}^{m} a_i\right\|^2 \leq m \sum_{i=1}^{m} \|a_i\|^2.
\end{equation}

Now, applying the assumption that the gradient dissimilarity is bounded, we can upper-bound the second term on the right-hand side of \eqref{eq:100} as:

\begin{equation}
    \begin{aligned}
    \label{eq:101}
    &\mathbb{E}[\mathbb{E}_{k \in \mathcal{S}_t}[\|\widehat{g}_t\|^2]] 
    \\&\leq \left(\frac{B_1+K}{K^2}\right) \left[\sum_{k=1}^{K} \frac{n_k}{n_{\mathcal{S}_t}} \|\nabla_{\theta_{t}^k} \widehat{L}_k(\theta_{t}, x_{i}^k)\|^2\right] + \frac{B_2^2}{K} \\
    &\leq \lambda \left(\frac{B_1+K}{K^2}\right) \left\|\sum_{k=1}^{K} \frac{n_k}{n_{\mathcal{S}_t}} \nabla_{\theta_{t}^k} \widehat{L}_k(\theta_{t}, x_{i}^k)\right\|^2 + \frac{B_2^2}{K}.
    \end{aligned}
\end{equation}

where, 

\begin{equation*}
\begin{aligned}
    \nabla_{\theta_{t}^k} \widehat{L}_k(\theta_{t}, x_{i}^k) = \nabla_{\theta_{t}^k} L_k(\theta_{t}, x_{i}^k) + \mathcal{N}\left(0, \sigma^2 \right),
\end{aligned}
\end{equation*}

and $\lambda$ represents the upper bound on weighted gradient diversity, given by:

\begin{equation}
    \frac{\sum_{k=1}^{K} \frac{n_k}{n_{\mathcal{S}_t}} \|\nabla_{\theta_{t}^k} \widehat{L}_k(\theta_{t}, x_{i}^k)\|_2^2}{\left\|\sum_{k=1}^{K} \frac{n_k}{n_{\mathcal{S}_t}} \nabla_{\theta_{t}^k} \widehat{L}_k(\theta_{t}, x_{i}^k)\right\|_2^2} \leq \lambda.
\end{equation}

Now, we proceed to bound the first term in \eqref{eq:100}.

We define the **average local stochastic gradient** at iteration $t$ as:
\begin{equation}
    \widehat{g}^{(t)} = \frac{1}{K} \sum_{k=1}^{K} \widehat{g}^{(t)}_k.
\end{equation}

Thus, we express:
\begin{equation}
    \begin{aligned}
        -\mathbb{E}_{i \in L_n} \mathbb{E}_{k \in \mathcal{S}_t} \left[ \left\langle \nabla L(\theta_t), \widehat{g}^{(t)} \right\rangle \right] 
        \\= -\mathbb{E}_{i \in L_n} \mathbb{E}_{k \in \mathcal{S}_t} \left[ \left\langle \nabla L(\theta_t), \frac{1}{K} \sum_{k=1}^{K} \widehat{g}^{(t)}_j \right\rangle \right].
    \end{aligned}
\end{equation}

Since the selection of devices occurs prior to computing the stochastic mini-batch gradients and each round of communication involves independently selected devices, we can rewrite the expectation order as:

\begin{equation}
    \begin{aligned}
        -\mathbb{E}_{i \in L_n} \mathbb{E}_{k \in \mathcal{S}_t} \left[ \left\langle \nabla L(\theta_t), \frac{1}{K} \sum_{k=1}^{K} \widehat{g}^{(t)}_j \right\rangle \right] 
        \\= -\mathbb{E}_{k \in \mathcal{S}_t} \mathbb{E}_{i \in L_n} \left[ \left\langle \nabla L(\theta_t), \frac{1}{K} \sum_{k=1}^{K} \widehat{g}^{(t)}_j \right\rangle \right] 
        \\= -\left\langle \nabla L(\theta_t), \mathbb{E}_{k \in \mathcal{S}_t} \left[ \frac{1}{K} \sum_{k=1}^{K} \mathbb{E}_t \left[\widehat{g}_j \right] \right] \right\rangle
        \\= -\left\langle \nabla L(\theta_t), \mathbb{E}_{k \in \mathcal{S}_t} \left[ \frac{1}{K} \sum_{k=1}^{K} \nabla_{\theta_{t}^k} \widehat{L}_k(\theta_{t}, x_{i}^k) \right] \right\rangle
        \\= -\left\langle \nabla L(\theta_t), \frac{1}{K} \mathbb{E}_{k \in \mathcal{S}_t} \left[ \sum_{k=1}^{K} \nabla_{\theta_{t}^k} \widehat{L}_k(\theta_{t}, x_{i}^k) \right] \right\rangle
        \\= -\left\langle \nabla L(\theta_t), \frac{1}{K} \left[ K \sum_{k=1}^{K} \frac{n_k}{n_{\mathcal{S}_t}} \nabla_{\theta_{t}^k} \widehat{L}_k(\theta_{t}, x_{i}^k) \right] \right\rangle
        \\= -\left\langle \nabla L(\theta_t), \sum_{k=1}^{K} \frac{n_k}{n_{\mathcal{S}_t}} \nabla_{\theta_{t}^k} \widehat{L}_k(\theta_{t}, x_{i}^k) \right\rangle.
    \end{aligned}
\end{equation}

Applying the identity:
\begin{equation}
    2\langle a, b \rangle = \|a\|^2 + \|b\|^2 - \|a - b\|^2,
\end{equation}
we derive:
\begin{equation}
    \begin{aligned}
        -\left\langle \nabla L(\theta_t), \sum_{k=1}^{K} \frac{n_k}{n_{\mathcal{S}_t}} \nabla_{\theta_{t}^k} \widehat{L}_k(\theta_{t}, x_{i}^k) \right\rangle
        \\= \frac{1}{2} \biggl[ -\|\nabla L(\theta_t)\|^2 - \left\|\sum_{k=1}^{K} \frac{n_k}{n_{\mathcal{S}_t}} \nabla_{\theta_{t}^k} \widehat{L}_k(\theta_{t}, x_{i}^k) \right\|^2 
        \\+ \left\|\nabla L(\theta_t) - \sum_{k=1}^{K} \frac{n_k}{n_{\mathcal{S}_t}} \nabla_{\theta_{t}^k} \widehat{L}_k(\theta_{t}, x_{i}^k) \right\|^2 \biggr]
        \\= \frac{1}{2} \biggl[ -\|\nabla L(\theta_t)\|^2 - \left\|\sum_{k=1}^{K} \frac{n_k}{n_{\mathcal{S}_t}} \nabla_{\theta_{t}^k} \widehat{L}_k(\theta_{t}, x_{i}^k) \right\|^2 
        \\+ \sum_{k=1}^{K} \frac{n_k}{n_{\mathcal{S}_t}} \|\nabla_{\theta_t} \widehat{L}_k(\theta_t, x_{i}^k) - \nabla_{\theta_{t}^k} \widehat{L}_k(\theta_t, x_{i}^k)\|^2 \biggr].
    \end{aligned}
\end{equation}

Since $L_k(\theta)$ is Lipschitz continuous, we obtain:
\begin{equation}
    \begin{aligned}
        \frac{1}{2} \biggl[ -\|\nabla L(\theta_t)\|^2 - \left\|\sum_{k=1}^{K} \frac{n_k}{n_{\mathcal{S}_t}} \nabla_{\theta_{t}^k} \widehat{L}_k(\theta_{t}, x_{i}^k) \right\|^2 
        \\+ \sum_{k=1}^{K} \frac{n_k}{n_{\mathcal{S}_t}} M^2 \|\theta_t - \theta_t^k\|^2 \biggr].
    \end{aligned}
\end{equation}

Thus, the first term in \eqref{eq:100} is bounded as follows:
\begin{equation}
\label{eq:102}
    \begin{aligned}
        -\eta_t\mathbb{E}[\mathbb{E}_{k \in \mathcal{S}_t}[\langle\nabla L(\theta_t),\widehat{g}_t\rangle]] 
        \\ \leq -\frac{\eta_t}{2} \|\nabla L(\theta_t)\|^2 - \frac{\eta_t}{2} \left\|\sum_{k=1}^{K} \frac{n_k}{n_{\mathcal{S}_t}} \nabla_{\theta_{t}^k} \widehat{L}_k(\theta_{t}, x_{i}^k) \right\|^2 
        \\+ \frac{\eta_t M^2}{2} \sum_{k=1}^{K} \frac{n_k}{n_{\mathcal{S}_t}} \|\theta_t - \theta_t^k\|^2.
    \end{aligned}
\end{equation}

Applying the Polyak-Łojasiewicz (PL) property, we obtain:
\begin{equation}
\label{eq:103}
    \begin{aligned}
        -\eta_t \mathbb{E} \left[ \mathbb{E}_{k \in \mathcal{S}_t} \left[ \left\langle \nabla L(\theta_t), \widehat{g}^{(t)} \right\rangle \right] \right] 
        \\ \leq -\mu \eta_t (L(\theta_t) - L(\theta^*)) - \frac{\eta_t}{2} \left\|\sum_{k=1}^{K} \frac{n_k}{n_{\mathcal{S}_t}} \nabla_{\theta_{t}^k} \widehat{L}_k(\theta_{t}, x_{i}^k) \right\|^2 
        \\+ \frac{\eta_t L^2}{2} \sum_{k=1}^{K} \frac{n_k}{n_{\mathcal{S}_t}} \|\theta_t - \theta_t^k\|^2.
    \end{aligned}
\end{equation}

We now proceed to simplify the last term on the right-hand side of \eqref{eq:103}.

Define \( t_c \triangleq \left\lfloor \frac{t}{|N|} \right\rfloor |N| \). Based on Algorithm \ref{alg1}, we have:
\begin{equation}
    \theta_{t_c+1} = \frac{1}{K} \sum_{k=1}^{K} \theta^k_{t_c+1}.
\end{equation}

Using this definition, the update rule of Algorithm \ref{alg1} can be rewritten as:
\begin{equation}
    \begin{aligned}
        \theta_t^k &= \theta^k_{t-1} - \eta_{t-1} \widehat{g}^k_{t-1} \\
        &= \theta^k_{t-2} - \left(\eta_{t-2} \widehat{g}^k_{t-2} + \eta_{t-1} \widehat{g}^k_{t-1} \right) \\
        &= \theta_{t_c+1} - \sum_{j=t_c+1}^{t-1} \eta_j \widehat{g}^k_j.
    \end{aligned}
\end{equation}

This follows directly from the iterative update rule. Given this, we can now compute the global model update:
\begin{equation}
    \theta_t = \theta_{t_c+1} - \frac{1}{K} \sum_{k=1}^{K} \sum_{j=t_c+1}^{t-1} \eta_j \widehat{g}_j^k.
\end{equation}

Next, our objective is to bound the term \( \mathbb{E}\|\theta_t - \theta_t^k\|^2 \) for \( t_c + 1 \leq t \leq t_c + E \), where \( n \) represents the indices of local updates. To achieve this, we relate this quantity to the variance between the stochastic gradient and the full gradient:

\begin{equation*}
    \begin{aligned}
        &\mathbb{E}\left[\|\theta_{(t_c+n)} - \theta^k_{(t_c+n)}\|^2\right] \\&= \mathbb{E} \left[ \left\|\theta_{t_c+1} - \sum_{j=t_c+1}^{t-1} \eta_j \widehat{g}^k_j 
        - \theta_{t_c+1} + \frac{1}{K} \sum_{k=1}^{K} \sum_{j=t_c+1}^{t-1} \eta_j \widehat{g}_j^{k}\right\|^2 \right] \\
        &= \mathbb{E} \left[ \left\|\sum_{j=1}^{n} \eta_{t_c+j} \widehat{g}^k_{t_c+j} - \frac{1}{K} \sum_{k=1}^{K} \sum_{j=1}^{n} \eta_{t_c+j} \widehat{g}_j^{t_c+j}\right\|^2 \right] \\
        &\leq 2 \mathbb{E} \left[ \left\|\sum_{j=1}^{n} \eta_{t_c+j} \widehat{g}^k_{t_c+j}\right\|^2 \right] 
        + 2 \mathbb{E} \left[ \left\|\frac{1}{K} \sum_{k=1}^{K} \sum_{j=1}^{n} \eta_{t_c+j} \widehat{g}_j^{t_c+j}\right\|^2 \right].
    \end{aligned}
\end{equation*}

Now, expanding the expectation terms, we obtain:
\begin{equation}
    \begin{aligned}
        =&2 \Bigg(\mathbb{E} \left[ \left\|\sum_{j=1}^{n} \eta_{t_c+j} \widehat{g}^k_{t_c+j} - \mathbb{E} \left[\sum_{j=1}^{n} \eta_{t_c+j} \widehat{g}^k_{t_c+j} \right]\right\|^2 \right] 
        \\&+ \left\|\mathbb{E} \left[\sum_{j=1}^{n} \eta_{t_c+j} \widehat{g}^k_{t_c+j} \right]\right\|^2 \Bigg) \\
        &+ 2 \mathbb{E} \biggl[\biggl\|\frac{1}{K} \sum_{k=1}^{K} \sum_{j=1}^{n} \eta_{t_c+j} \widehat{g}_j^{t_c+j} 
        \\&- \mathbb{E} \left[\frac{1}{K} \sum_{k=1}^{K} \sum_{j=1}^{n} \eta_{t_c+j} \widehat{g}_j^{t_c+j} \right] \biggr\|^2 \biggr] \\
        &+ \left\|\mathbb{E} \left[\frac{1}{K} \sum_{k=1}^{K} \sum_{j=1}^{n} \eta_{t_c+j} \widehat{g}_j^{t_c+j} \right]\right\|^2.
    \end{aligned}
\end{equation}

Since the gradient estimates follow an unbiased assumption, we can simplify further:
\begin{equation}
    \begin{aligned}
        &= 2 \mathbb{E} \left[ \left\|\sum_{j=1}^{n} \eta_{t_c+j} \left(\widehat{g}^k_{t_c+j} - g^k_{t_c+j} \right) \right\|^2 \right] 
        + \left\|\sum_{j=1}^{n} \eta_{t_c+j} g^k_{t_c+j}\right\|^2 \\
        &\quad + 2 \mathbb{E} \left[ \left\|\frac{1}{K} \sum_{k=1}^{K} \sum_{j=1}^{n} \eta_{t_c+j} \left(\widehat{g}_j^{t_c+j} - g_j^{t_c+j} \right) \right\|^2 \right] 
        \\&\quad + \left\|\frac{1}{K} \sum_{k=1}^{K} \sum_{j=1}^{n} \eta_{t_c+j} g_j^{t_c+j} \right\|^2.
    \end{aligned}
\end{equation}

The above inequalities hold due to the convexity properties and unbiased nature of the stochastic gradients.

Using the **smoothness assumption** along with the i.i.d. sampling property, we can derive the following bound:

\begin{equation}
\begin{aligned}
    \mathbb{E}\left[\|\theta_{(t_c+n)} - \theta^k_{(t_c+n)}\|^2\right] 
    \leq 2\mathbb{E} \Bigg[ \sum_{j=1}^{n} \eta^2_{t_c+j} \|g_{t_c+j}^k - g_{t_c+j}\|^2 \\ 
    + \sum_{j \neq u \lor k \neq v} \left\langle \eta_j g_{t_c+j}^k - \eta_j g_{t_c+j}, \eta_u g_{t_c+u}^v - \eta_u g_{t_c+u} \right\rangle \\
    + \biggl\| \sum_{j=1}^{n} \eta_{t_c+j} g_{t_c+j} \biggr\|^2 
    + \frac{1}{K^2} \sum_{k \in \mathcal{S}_t} \sum_{j=1}^{n} \eta^2_{t_c+j} \|g_{t_c+j}^j - g_{t_c+j}\|^2 \\
    + \frac{1}{K^2} \sum_{j \neq u \lor k \neq v} \left\langle \eta_j g_{t_c+j}^j - \eta_j g_{t_c+j}, \eta_u g_{t_c+u}^v - \eta_u g_{t_c+u} \right\rangle \\
    + \left\| \frac{1}{K} \sum_{k \in \mathcal{S}_t} \sum_{j=1}^{n} \eta_{t_c+j} g_{t_c+j}^j \right\|^2 \Bigg].
\end{aligned}
\end{equation}

Applying the convexity properties, we further simplify:

\begin{equation}
\begin{aligned}
    &\mathbb{E}\left[\|\theta_t - \theta^k_t\|^2\right] 
    \leq \\&2\mathbb{E} \Bigg[ \sum_{j=1}^{n} \eta^2_{t_c+j} \|g_{t_c+j}^k - g_{t_c+j}\|^2 
    + n \sum_{j=1}^{n} \eta^2_{t_c+j} \|g_{t_c+j}\|^2 \\&
    + \frac{1}{K^2} \sum_{k \in \mathcal{S}_t} \sum_{j=1}^{n} \eta^2_{t_c+j} \|g_{t_c+j}^j - g_{t_c+j}\|^2 
    \\&+ \frac{n}{K^2} \sum_{k \in \mathcal{S}_t} \sum_{j=1}^{n} \eta^2_{t_c+j} \|g_{t_c+j}^j\|^2 \Bigg].
\end{aligned}
\end{equation}

Now, leveraging **Assumption 1**, we obtain:

\begin{equation}
\begin{aligned}
    \mathbb{E}&\left[\|\theta_t - \theta^k_{t}\|^2\right] \leq 2 \Bigg[ \sum_{j=1}^{n} \eta^2_{t_c+j} \left( B_1 \| g^k_{(t_c+j)} \|^2 + \frac{B_2^2}{K} \right) 
    \\&+ n \sum_{j=1}^{n} \eta^2_{t_c+j} \| g^k_{(t_c+j)} \|^2 \Bigg]    \\ &+ \frac{1}{K^2} \sum_{k \in \mathcal{S}_t} \sum_{j=1}^{n} \eta^2_{t_c+j} \biggl( B_1 \| g_j^{(t_c+j)} \|^2 + \frac{B_2^2}{K} \biggr) 
    \\&+ \frac{n}{K^2} \sum_{k \in \mathcal{S}_t} \sum_{j=1}^{n} \eta^2_{t_c+j} \| g_j^{(t_c+j)} \|^2.
\end{aligned}
\end{equation}

Summing over the participating clients, we derive:

\begin{equation}
\begin{aligned}
    &\mathbb{E} \sum_{k \in \mathcal{S}_t} \|\theta_t - \theta^k_t\|^2 
    \leq \\&2 \Bigg[ \sum_{k \in \mathcal{S}_t} \sum_{j=1}^{n} \eta^2_{t_c+j} B_1 \| g^k_{(t_c+j)} \|^2 
    + \sum_{j=1}^{n} \eta^2_{t_c+j} \frac{B_2^2}{K} \\&
    + n \sum_{k \in \mathcal{S}_t} \sum_{j=1}^{n} \eta^2_{t_c+j} \| g^k_{(t_c+j)} \|^2 \Bigg] 
    \\&+ \frac{1}{K} \sum_{k \in \mathcal{S}_t} \sum_{j=1}^{n} \eta^2_{t_c+j} B_1 \| g_j^{(t_c+j)} \|^2 
    + \sum_{j=1}^{n} \eta^2_{t_c+j} \frac{B_2^2}{K^2}.
\end{aligned}
\end{equation}

Using the **weighted gradient diversity upper bound**, we obtain:

\begin{equation}
\label{eq:104}
\begin{aligned}
    &\mathbb{E} \sum_{k=1}^{K} \frac{n_k}{n_{\mathcal{S}_t}} \|\theta_t - \theta^k_t\|^2 
    \leq \\&2 \Bigg( \frac{K+1}{K} \Bigg) \Bigg( \left[ B_1+N\right] \sum_{j=t_c+1}^{t-1} \eta_j^2 \sum_{k=1}^{K} \frac{n_k}{n_{\mathcal{S}_t}} \|\nabla_{\theta_{t}^k} \widehat{L}_k(\theta_t, x_{i}^k)\|^2 \\&
    \quad\quad \quad \quad \quad \quad \quad \quad \quad \quad \quad \quad \quad \quad \quad \quad \quad  + \sum_{j=t_c+1}^{t-1} \frac{\eta_j^2 B_2^2}{K} \Bigg).
\end{aligned}
\end{equation}

Finally, substituting the results from \eqref{eq:101}, \eqref{eq:102}, and \eqref{eq:104} into \eqref{eq:100}, we obtain:

\begin{equation}
\label{eq:105}
\begin{aligned}
    &\mathbb{E}[L(\theta_{t+1})] - L(\theta^*) 
    \leq (1-\mu\eta_t) \mathbb{E}[L(\theta_t)-L(\theta^*)] + \frac{M\eta_t^2 B_2^2}{2K} \\&
    + \frac{\eta_t M^2}{K} \sum_{j=t_c+1}^{t-1} \eta_j^2 \frac{(K+1)B_2^2}{K} 
    \\&+ \frac{\eta_t}{2} \left[-1+\frac{M\lambda\eta_t(B_1+K)}{K} \right] 
    \left\| \sum_{k=1}^{K} \frac{n_k}{n_{\mathcal{S}_t}} \nabla_{\theta_{t}^k} \widehat{L}_k(\theta_t, x_{i}^k) \right\|^2 \\&
    + \frac{\eta_t M^2 (K+1)}{K^2} \Bigg[ \lambda(B_1+N) \sum_{j=t_c+1}^{t-1} \eta_j^2 
    \\&\quad\quad \quad \quad \quad \quad \quad\quad \quad \quad \quad \quad\left\| \sum_{k=1}^{K} \frac{n_k}{n_{\mathcal{S}_t}} \nabla_{\theta_{t}^k} \widehat{L}_k(\theta_t, x_{i}^k) \right\|^2 \Bigg].
\end{aligned}
\end{equation}

To streamline our convergence analysis, we define the following parameters:

\begin{equation}
    \triangle_t = 1 - \mu \eta_t,
\end{equation}
\begin{equation}
    c_t = \frac{\eta_t M B_2^2}{K} \left[\frac{\eta_t}{2} + \frac{M(K+1)}{K} \sum_{j=t_c+1}^{t-1} \eta_j^2 \right],
\end{equation}
\begin{equation}
    B_t = \frac{\lambda (K+1) \eta_t M^2}{K^2} (B_1 + N).
\end{equation}

With these definitions in place, we can now rewrite and simplify the bound obtained in \eqref{eq:105} as:

\begin{equation}
\label{eq:106}
    \begin{aligned}
        &\mathbb{E}[L(\theta_{t+1})] - L(\theta^*) 
        \leq \triangle_t \mathbb{E}[L(\theta_t) - L(\theta^*)] + c_t \\
        & + \frac{\eta_t}{2} \left[ -1 + \lambda M \eta_t \left(\frac{B_1 + K}{K}\right) \right] 
        \left\| \sum_{k=1}^{K} \frac{n_k}{n_{\mathcal{S}_t}} \nabla_{\theta_t^k} \widehat{L}_k(\theta_t, x_i^k) \right\|^2 \\
        &+ B_t \sum_{j=t_c+1}^{t-1} \eta_j^2 
        \left\| \sum_{k=1}^{K} \frac{n_k}{n_{\mathcal{S}_t}} \nabla_{\theta_t^k} \widehat{L}_k(\theta_t, x_i^k) \right\|^2.
    \end{aligned}
\end{equation}

This final bound characterizes the expected difference between the global loss function and the optimal loss at iteration $t+1$, highlighting the influence of the learning rate $\eta_t$, the gradient dissimilarity parameter $B_1$, and the noise variance $B_2^2$.

\end{proof}
\end{document}